\documentclass{article} %
\usepackage[T1]{fontenc}
\usepackage[preprint]{iclr2027_conference}
\usepackage{times}

\usepackage{amsthm,amsfonts,amssymb,amsmath,graphicx} 

\usepackage{amsmath,amsfonts,bm}

\def\eqref#1{equation~\ref{#1}}

\def\1{\bm{1}}

\DeclareMathAlphabet{\mathsfit}{\encodingdefault}{\sfdefault}{m}{sl}
\SetMathAlphabet{\mathsfit}{bold}{\encodingdefault}{\sfdefault}{bx}{n}

\usepackage{hyperref}
\usepackage{url}
\newtheorem{lemma}{Lemma}
\usepackage{booktabs}
\usepackage{graphicx}
\usepackage{float}
\usepackage{tikz}

\usepackage{enumerate}
\usepackage[noabbrev,nameinlink,capitalise]{cleveref}
\creflabelformat{equation}{#2\textup{#1}#3}

\usepackage{glossaries}
\glsdisablehyper
\newacronym{pinn}{PINN}{Physics-Informed neural network}
\newacronym{pde}{PDE}{Partial Differential Equation}
\newacronym{pihnn}{PIHNN}{Physics-Informed Holomorphic Neural Network}
\newacronym{mlp}{MLP}{MultiLayer Perceptron}

\title{Enriched Holomorphic Neural Networks\\for 2D and 3D Laplace Problems}

\author{Kalp Shah$^{1}$\quad Andrei Vlad Nicula$^{1}$\quad Mihnea Andrei B\^{a}rsan$^{1}$\quad Robert-\textcommabelow{S}tefan Sofroni$^{1}$\\
\textbf{Agneya Bhardwaj$^{1}$\quad Jan Jel\'{\i}nek$^{1}$\quad Alejandro Valverde Mahou$^{2}$}\\
$^{1}$University of Amsterdam\quad $^{2}$Technical University of Denmark\\
\texttt{\{kalp.shah, andrei.nicula, mihnea.barsan, robert.sofroni}\\
\texttt{\{agneya.bhardwaj, jan.jelinek\}@student.uva.nl}\quad
\texttt{avama@dtu.dk}}

\begin{document}

\maketitle

\begin{abstract}

Boundary-value problems for the Laplace equation develop singular behaviour near corners, edges and junctions that smooth approximators resolve poorly. We combine holomorphic neural networks with singular functions derived from the local geometry and boundary conditions. The representation is harmonic by construction, so no interior residual is needed: the network fits the smooth remainder from boundary data alone, and the enrichment terms capture the singular behaviour. Singular exponents are trainable, so the construction also applies when the local expansion is not known in advance. We extend the approach to three dimensions using independent networks evaluated along complex null directions. The enriched model is the most accurate tested method on every problem. Its margin is several orders of magnitude where the enrichment captures the singularity exactly, and a factor of about 1.7--2.1 over the next-best method on a CdZnTe radiation detector with many singular sites, in both two and three dimensions.
\end{abstract}

\section{Introduction}
\label{sec:introduction}

Solving \glspl{pde} is difficult on domains with non-trivial geometries or boundary conditions. \glspl{pinn} \citep{raissi2019} offer an alternative by incorporating the governing equations directly into the training objective. The solution is represented by a neural network, so it is continuous and its derivatives come from automatic differentiation. The method is therefore mesh-free and differentiable end-to-end, which motivates its use for irregular geometries and inverse problems \citep{luo2025}.

Unfortunately, the resulting objective suffers from competing loss terms, ill-conditioning and imbalanced gradients, which slow convergence or degrade accuracy \citep{wang2021,wang2022,rahaman2019}. Much work has therefore improved the optimization and representation of PINNs. Adaptive loss weighting balances the terms of the objective \citep{kendall2018,heydari2019,wang2022,mcclenny2023}, and adaptive collocation changes where the residual is evaluated \citep{wu2023}. Other approaches change the representation, through periodic activations or Fourier features \citep{sitzmann2020,liu2025}, or by replacing the multilayer perceptron with a Kolmogorov--Arnold network \citep{liu2024kan}. Domain-decomposition approaches such as conservative PINNs (cPINNs) \citep{jagtap2020conservative} and extended PINNs (XPINNs) \citep{jagtap2020extended} partition the domain into subdomains coupled through interface conditions. %

We address the representation side of these difficulties. We focus on the Laplace equation, $\Delta u = 0$, and build on the \glspl{pihnn} formulations introduced by \citet{calafa2024} for two-dimensional problems and further developed by \citet{ballini2026} in three dimensions. These representations have also been shown to benefit from enrichment with known singularities for specific problems, such as crack tips in plane elasticity \citep{calafa2025}. Our approach addresses the enrichment for a set of local boundary configurations rather than for one particular problem: we construct a dictionary mapping each local boundary configuration to its enrichment family, together with a rule for orienting the branch cuts so that the representation remains single-valued inside the domain. We develop the construction in two and three dimensions.

\section{Related work}
\label{sec:related}

Representations that satisfy the governing equation in the interior reduce a boundary value problem to determining model parameters from boundary information. Classical examples include the boundary element method, the method of fundamental solutions \citep{fairweather1998}, harmonic kernel collocation \citep{schaback2008}, and rational approximations such as the lightning Laplace solver \citep{gopal2019}. These methods construct the solution from fixed analytical families, such as harmonic polynomials or fundamental solutions centred outside the domain \citep{axler2001}. Neural representations instead learn part of the approximating family, at the cost of a nonconvex fit. \citet{ghosh2023} represent harmonic functions as the real part of a complex-valued neural network, while \citet{calafa2024} develop \glspl{pihnn} for two-dimensional problems. In 3D, where the correspondence with complex analysis is lost, \citet{ballini2026} write the solution as an integral over projection directions, approximate it by a weighted sum over a finite set of directions, and share one network across them.

Exact satisfaction of the governing equation, however, does not guarantee an efficient representation of solutions with reduced regularity. At corners and boundary-condition junctions, local singularities arise whose exponents depend on the geometry and boundary conditions \citep{grisvard1985}. Enriching the approximation space with known singular functions is a classical remedy in numerical \gls{pde} methods \citep{fix1973,li2000}. Among neural solvers, the singularity-enriched \gls{pinn} (SEPINN) of \citet{hu2024} augments a real-valued \gls{pinn} with singular functions for polygonal corners, mixed boundary conditions and edges in 3D, while \citet{calafa2025} enrich \glspl{pihnn} with the square-root term from Williams' expansion at crack tips. In these approaches, the singularity and its exponent are prescribed for the problem at hand.

\section{Method}
\label{sec:method}

\subsection{Problem formulation}

We consider the Laplace boundary-value problem on a bounded domain
$\Omega\subset\mathbb{R}^d$, with $d\in\{2,3\}$. The boundary is partitioned into disjoint Dirichlet and Neumann portions, $\partial\Omega=\Gamma_D\cup\Gamma_N$, and the solution $u:\Omega\rightarrow\mathbb{R}$ satisfies
\begin{equation}
    \begin{aligned}
    \Delta u(\bm{x}) &= 0, && \bm{x}\in\Omega,\\
    u(\bm{x}) &= g_D(\bm{x}), && \bm{x}\in\Gamma_D,\\
    \nabla u(\bm{x})\cdot\bm{n}(\bm{x}) &= g_N(\bm{x}), && \bm{x}\in\Gamma_N,
    \end{aligned}
    \label{eq:laplace-bvp}
\end{equation}
where $g_D$ and $g_N$ denote the prescribed Dirichlet and Neumann data, respectively, and $\bm{n}$ is the outward unit normal. We assume that $\Gamma_D$ has positive measure, so that the solution is unique.

We represent the boundary as a collection of segments in two dimensions or faces in three dimensions. We refer to points at which adjacent boundary pieces meet, or at which the boundary condition or its prescribed value changes, as \emph{junctions}. These local boundary configurations determine the singular terms used to enrich the neural representation.

\subsection{Holomorphic functions}

A function $f:U\to\mathbb C$ on an open set $U\subset\mathbb C$ is holomorphic if it is complex differentiable at every point of $U$. Writing $z=x+iy$ and $f(z)=p(x,y)+iq(x,y)$, and using the subscript notation for partial derivatives throughout so that $p_x=\frac{\partial p}{\partial x}$, for continuously differentiable $p$ and $q$, complex differentiability is equivalent to the Cauchy--Riemann equations $p_x=q_y$ and $p_y=-q_x$ \citep{axler2001, evans2010}. Differentiating these once more gives $\Delta p = p_{xx}+p_{yy} = q_{yx}-q_{xy} = 0$, and the same for $q$, so both parts of a holomorphic function are harmonic. Conversely, on a simply connected domain every harmonic function is the real part of a holomorphic one. A function that is holomorphic on all of $\mathbb{C}$ is called \emph{entire}.

Two properties of holomorphic functions are central to this work. First, they are closed under composition and linear combination: a network assembled from holomorphic layers is therefore holomorphic, and adding further holomorphic terms to its output keeps it so. Second, the real part of any holomorphic function is harmonic so a holomorphic network represents  a harmonic function for every value of its weights. Together these mean the network is harmonic by construction, with no interior residual required. The Cauchy--Riemann equations also give the complex derivative as $f'(z)=p_x-ip_y$, from which the normal 
derivative follows directly.

\subsection{Enrichment functions}
\label{sec:enrichment}

Near corners and junctions, solutions to the Laplace equation develop singular behaviour that smooth approximators resolve poorly: the solution may be bounded, but its gradient can blow up. Rather than relying on the
network to learn this from boundary samples, we augment the representation with enrichment functions drawn from local singular expansions for the given geometry and boundary conditions \citep{grisvard1985,hu2024,calafa2025}.

Consider a planar wedge with interior angle $\omega$ and polar coordinates $(r,\theta)$ centred at its vertex. For a homogeneous Dirichlet condition at $\theta=0$ and a homogeneous Neumann condition at $\theta=\omega$, the leading singular mode is $r^\lambda\sin(\lambda\theta)$ with $\lambda=\pi/(2\omega)$. The solution stays bounded as $r\to 0$, but its gradient grows as $r^{\lambda-1}$ and diverges whenever $\lambda<1$, i.e. when $\omega>\pi/2$. This is illustrated in \cref{fig:enrichment} for $\omega=2\pi/3$, $\lambda=3/4$, where the field looks smooth but the gradient concentrates at the vertex.

In the local complex coordinate $w=re^{i\theta}$, the singular modes are the real and imaginary parts of $ w^\lambda = \exp(\lambda\operatorname{Log}w)$.
We place the branch cut of $\operatorname{Log}$ along a fixed exterior direction of the junction (\cref{app:harmonicity}). The complex coefficient of each term then absorbs the phase of the wall from which $\theta$ is measured and the order of the two conditions.
At each junction there is not one singular exponent but a countably infinite sequence $\lambda_1<\lambda_2<\cdots$, each indexing a valid local mode. For a D$\to$N (Dirichlet to Neumann) junction at angle $\omega$, for instance, $\lambda_k=(2k-1)\pi/(2\omega)$. The leading term ($k=1$) is the most singular and dominates near the vertex, and higher terms are progressively smoother. The number of terms per junction is a hyperparameter we call the \emph{ladder depth}. We use ladder depth one except on Motz's problem and the cracked beam, where the first four exponents are included (\cref{app:training-2d}). \Cref{tab:dictionary} lists the enrichment terms.

\begin{figure}[ht]
  \centering
  \includegraphics[width=\linewidth]{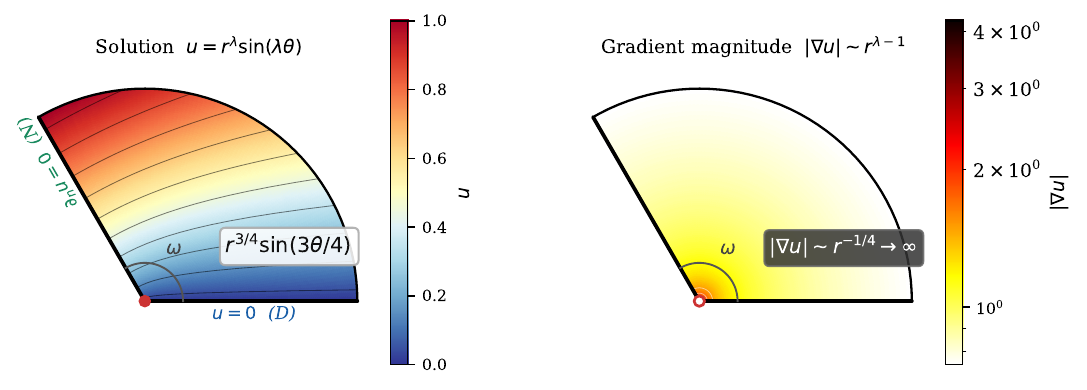}
  \caption{Singular behaviour at a D$\to$N junction with $\omega=2\pi/3$,
  $\lambda=3/4$. Left: the leading mode $r^\lambda\sin(\lambda\theta)$ is
  bounded at the vertex. Right: the gradient magnitude $|\nabla u|\sim
  r^{\lambda-1}=r^{-1/4}$ diverges there. The enrichment term $w^\lambda$
  represents this mode directly.}
  \label{fig:enrichment}
\end{figure}

\begin{table}[ht]
  \caption{Enrichment dictionary. Each junction matches one row and contributes the term $E_\lambda$ with a trainable coefficient $c_q$. Local coordinates are centred at the junction $z_c$, with $w=z-z_c=re^{i\theta}$; $\omega$ is the interior angle and $k=1,2,\dots$ indexes the ladder. Every term is holomorphic, so the representation remains exactly harmonic. \cref{fig:sitekind_classification} is a visualization of this table.}
  \label{tab:dictionary}
  \centering
  \small
  \begin{tabular}{llll}
    \toprule
    Junction & Local solution & $|\nabla u|$ & Term $E_\lambda$ \\
    \midrule
    Dirichlet to Neumann
      & $r^{\lambda}\sin(\lambda\theta)$, $\lambda=(2k-1)\pi/(2\omega)$
      & $r^{\lambda-1}$ & $w^{\lambda}$ \\
    Both Dirichlet, value jump $\Delta g$
      & $(\Delta g/\omega)\,\theta$ & $r^{-1}$ & $\arg w$ \\
    Both Neumann, flux jump, $\omega\in\pi\mathbb{N}$
      & $\Re[w\operatorname{Log}w]$ & $|\ln r|$ & $w\operatorname{Log}w$ \\
    Both Neumann, equal flux, $\omega>\pi$
      & $r^{\lambda}\cos(\lambda\theta)$, $\lambda=k\pi/\omega$
      & $r^{\lambda-1}$ & $w^{\lambda}$ \\
    Both Dirichlet, equal values, $\omega>\pi$
      & $r^{\lambda}\sin(\lambda\theta)$, $\lambda=k\pi/\omega$
      & $r^{\lambda-1}$ & $w^{\lambda}$ \\
    Otherwise & $C^\infty$ away from vertex & bounded & none \\
    \bottomrule
  \end{tabular}
\end{table}

\paragraph{Trainable exponents.}
The exponents in \cref{tab:dictionary} and \cref{fig:sitekind_classification} can be fixed at their theoretical values or trained. Since $w^\lambda$ is holomorphic for every $\lambda>0$, varying the exponent does not affect harmonicity, but only the exponents of the table satisfy the local boundary conditions, so a trained $\lambda$ is a fitted parameter rather than the singular exponent itself. Each term carries one real parameter $\tau$, with $\lambda=\varepsilon+\operatorname{softplus}(\tau)$ initialised at
the predicted value, except the conical degree, which starts at $1/2$. This allows the construction to be applied when the exponent is not known in advance, whereas prior enrichment schemes fix $\lambda$ \citep{hu2024,calafa2025}.

\paragraph{Extension to 3D.}
For a straight edge with unit tangent $t$ and transverse orthonormal frame $e_1,e_2$, we set
\begin{equation}
 w=(\bm{x}-\bm{x}_e)\cdot(e_1+ie_2), \qquad E_\lambda(\bm{x})=w^\lambda.
 \label{eq:edge}
\end{equation}
The direction $e_1+ie_2$ satisfies $a\cdot a=0$, so $E_\lambda$ is harmonic wherever its branch is holomorphic. %

\subsection{Enriched Holomorphic neural networks}

\glspl{pihnn} exploit the compositional property of holomorphic functions to construct a holomorphic neural representation. Let $f_\eta:\mathbb{C}\to\mathbb{C}$ denote such a network, with complex-valued parameters $\eta$. Since holomorphic functions are closed under composition, $f_\eta$ is holomorphic by construction, and its real part is therefore harmonic \citep{ghosh2023,calafa2024}.

The \gls{pihnn} represents the remaining part of the solution, and $Q$ enrichment terms $E_{\lambda_q}$, one per singular site and ladder index, represent the behaviour near the sites. The output of the model is
\begin{equation}
 u_{\eta,\bm{c}}(x,y)=\Re\left[f_\eta(z)+\sum_{q=1}^Q c_q E_{\lambda_q}(z)\right],
\label{eq:model2d}
\end{equation}
where $c_q$ determine the amplitude and phase of each term $E_{\lambda_q}$.

Since the representation is harmonic by construction, no interior residual is
needed, and the parameters are fitted to the boundary conditions by minimizing
\begin{equation}
\mathcal{L}(\eta,\bm{c}) =
\frac{1}{|X_D|}
\sum_{\bm{x}\in X_D}
\left\|
u_{\eta,\bm{c}}(\bm{x})-g_D(\bm{x})
\right\|^2
+
\frac{1}{|X_N|}
\sum_{\bm{x}\in X_N}
\left\|
\partial_n u_{\eta,\bm{c}}(\bm{x})-g_N(\bm{x})
\right\|^2,
\label{eq:boundary-fit}
\end{equation}
where $X_D$ and $X_N$ denote the corresponding boundary sample sets, and
a term is omitted when the corresponding boundary condition is absent.

\paragraph{Extension to 3D}
\label{sec:extension-3d}

In 3D there is no single complex coordinate, so we use complex projections of the form $a\cdot\bm{x}$, where $a\in\mathbb{C}^3$. For a holomorphic function $f$ and a fixed direction $a$, the chain rule gives
\begin{equation}
\Delta_{\bm{x}} f(a\cdot\bm{x}) = (a\cdot a)f''(a\cdot\bm{x}),
~~~~~ a\cdot a=\sum_{k=1}^{3}a_k^2.
\label{eq:chain}
\end{equation}
Thus, choosing a null direction satisfying $a\cdot a=0$ makes the projected function harmonic.

We use $K$ such projections and sum their outputs:
\begin{equation}
u_{\bm{\eta},\bm{c}}(\bm{x}) =
\Re\left[
\sum_{k=1}^{K} f_{\eta_k}(\zeta_k)
+
\sum_{q=1}^{Q} c_q E_{\lambda_q}(\bm{x})
\right],
\label{eq:model-3d}
\end{equation}

where $\zeta_k=a_k\cdot\hat{\bm{x}}+b_k$, with a complex offset $b_k$ and $\hat{\bm{x}}=(\bm{x}-\bm{x}_0)/L_0$ (\cref{sec:model-3d}). Each branch is harmonic, and so is their sum.
The normal derivative follows directly from the chain rule. For an outward unit normal $\bm{n}\in\mathbb{R}^3$,
\begin{equation}
\partial_n u_{\bm{\eta},\bm{c}}(\bm{x})=
\Re\left[
\frac{1}{L_0}\sum_{k=1}^{K}
(\bm{n}\cdot a_k)f'_{\eta_k}(\zeta_k)
+
\sum_{q=1}^{Q}
c_q\bm{n}\cdot\nabla_{\bm{x}}E_{\lambda_q}(\bm{x})
\right].
\label{eq:neumann-3d}
\end{equation}
The model is fitted with the same boundary loss as in 2D.

Since each network branch is entire and the enrichment terms are harmonic away from their branch cuts, the full representation is harmonic in $\Omega$ (\cref{app:harmonicity}). For pure Dirichlet problems, $\Gamma_D=\partial\Omega$, this bounds the interior error by the boundary misfit. The error $e=u_{\bm{\eta},\bm{c}}-u^\star$ is harmonic, and the junctions $J$, points in 2D and edges in 3D, form a polar set. If $e$ is bounded and continuous up to $\partial\Omega\setminus J$, the extended maximum principle (\citealp[Thm.~3.6.9]{ransford1995}; \citealp{armitage2001} in 3D) gives
\begin{equation}
  \sup_{\Omega}\left|u_{\bm{\eta},\bm{c}}-u^\star\right|
  \;\le\;
  \sup_{\partial\Omega\setminus J}\left|u_{\bm{\eta},\bm{c}}-g_D\right|.
  \label{eq:certificate}
\end{equation}
The error is bounded when every enrichment term is, which rules out a $\ln|w|$ component in the value-jump terms (\cref{app:harmonicity}). We evaluate the right-hand side on a finite boundary sample, which can only underestimate the supremum, so the reported value is an error estimate rather than a certified bound.

\section{Model Setup}
\label{sec:setup}

In both dimensions the model has the same three parts: a holomorphic core, the enrichment terms that \cref{tab:dictionary} assigns to the junctions, and a real part that turns their sum into a potential. The core is a small complex-valued network in which each layer applies a complex affine map followed by the complex exponential,
\begin{equation}
 h_0=\zeta,\qquad h_{\ell+1}=\exp(W_\ell h_\ell+b_\ell),
 \qquad f_\eta(\zeta)=W_{\mathrm{out}}h_m+b_{\mathrm{out}},
\label{eq:holo-network}
\end{equation}
where $\zeta$ is a complex scalar, $m$ is the number of hidden layers, the exponential acts on each component, and $\eta$ collects the complex weights and biases. The exponential is entire, so $f_\eta$ is entire for any value of $\eta$.

\paragraph{2D model}
\label{sec:model-2d}

In two dimensions the input is the normalized coordinate $\hat z=(z-z_0)/L_0$, with $z_0$ the centre of the domain and $L_0$ its characteristic radius. The core has $m=2$ hidden layers of width 16 with complex exponential activations. The enrichment terms are evaluated in the physical coordinate, since their sites, cut directions and interior angles come from the geometry of the problem, and one term is placed per junction and ladder index. All 2D problems have unit length scale, so $w$ needs no separate reference length.

\paragraph{3D model}
\label{sec:model-3d}

The 3D model first projects the point onto several complex coordinates. \Cref{fig:model3d} shows how a point $\bm{x}$ becomes a potential $u(\bm{x})$.

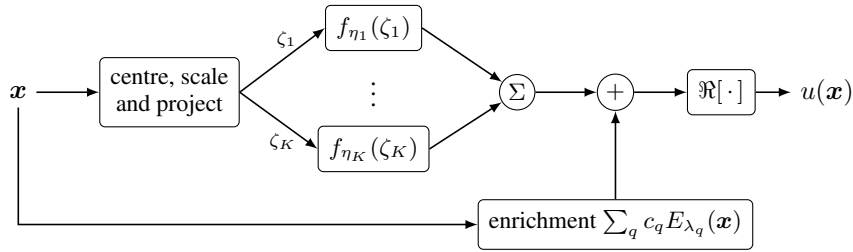
\begin{figure}[ht]
  \centering
   \scalebox{1}{
  \begin{tikzpicture}[
      box/.style={draw, rounded corners=2pt, align=center, font=\small, inner sep=4pt},
      op/.style={draw, circle, font=\small, inner sep=1.5pt},
      arr/.style={->, >=latex, semithick}]
    \node (x) at (0,0) {$\bm{x}$};
    \node[box] (proj) at (2.0,0) {centre, scale\\and project};
    \node[box] (f1) at (4.7,0.85) {$f_{\eta_1}(\zeta_1)$};
    \node (dots) at (4.7,0.08) {$\vdots$};
    \node[box] (fK) at (4.7,-0.75) {$f_{\eta_K}(\zeta_K)$};
    \node[op] (sum) at (6.6,0) {$\Sigma$};
    \node[op] (plus) at (7.9,0) {$+$};
    \node[box] (enr) at (7.9,-1.75) {enrichment $\sum_q c_q E_{\lambda_q}(\bm{x})$};
    \node[box] (re) at (9.3,0) {$\Re[\,\cdot\,]$};
    \node (u) at (10.7,0) {$u(\bm{x})$};
    \draw[arr] (x) -- (proj);
    \draw[arr] (proj.east) -- node[above, font=\scriptsize, pos=0.55] {$\zeta_1$} (f1.west);
    \draw[arr] (proj.east) -- node[below, font=\scriptsize, pos=0.55] {$\zeta_K$} (fK.west);
    \draw[arr] (f1.east) -- (sum);
    \draw[arr] (fK.east) -- (sum);
    \draw[arr] (sum) -- (plus);
    \draw[arr] (x.south) |- (enr.west);
    \draw[arr] (enr.north) -- (plus.south);
    \draw[arr] (plus) -- (re);
    \draw[arr] (re) -- (u);
  \end{tikzpicture}
  }
   \caption{Forward pass of the 3D model. The point is projected onto $K$ complex coordinates $\zeta_k=a_k\cdot\hat{\bm{x}}+b_k$ along null directions, each coordinate is passed through its own holomorphic network, and the outputs are summed. The enrichment terms are evaluated at the physical point and added before the real part is taken.}
  \label{fig:model3d}
\end{figure}

The point is centred and scaled, $\hat{\bm{x}}=(\bm{x}-\bm{x}_0)/L_0$, and projected onto $K$ complex coordinates $\zeta_k=a_k\cdot\hat{\bm{x}}+b_k$, with a complex offset $b_k$. By \cref{eq:chain}, a holomorphic function of $\zeta_k$ is harmonic in $\bm{x}$ when $a_k\cdot a_k=0$, so we write each direction as
\begin{equation}
 a(\xi,\rho)=\rho\left(\frac{1-\xi^2}{2},\;\frac{i(1+\xi^2)}{2},\;\xi\right),
 \qquad \xi,\rho\in\mathbb{C},
\label{eq:null-direction}
\end{equation}
which is null for every $\xi$ and $\rho$. The directions can therefore be trained freely. The same length scale is used for all coordinates, since a separate scale per coordinate would break the null condition. Each $\zeta_k$ is passed through its own network of the form \cref{eq:holo-network}, with two hidden layers of width 8, and the $K=16$ outputs are added. This sum is a discrete counterpart of the integral over directions used by \citet{ballini2026}, with separate parameters for each direction instead of one network shared across them. The number and size of the branches are the same for all problems (\cref{app:architecture}).

The enrichment terms are evaluated at the physical point, since their location is part of the geometry. A straight edge uses the term $w^\lambda$ of \cref{eq:edge}, with the exponent that \cref{tab:dictionary} gives for its dihedral angle and the conditions on its two faces. A circular conical vertex uses the separated mode
\begin{equation}
 S_\nu(R,\vartheta)=R^\nu P_\nu(\cos\vartheta),
\label{eq:cone-3d}
\end{equation}
where $R$ is the distance from the vertex, $\vartheta$ the angle from the cone axis, and $\nu$ the first positive root of $P_\nu(\cos\vartheta_{\mathrm{wall}})=0$. A trained $\nu$ enters both factors, and $P_\nu$ is evaluated by a 400-term hypergeometric series whose tail is below $10^{-27}$ on the domain. Site positions, edge frames and branch cuts are fixed by the geometry and are not trained, since moving them could place a singular line inside the domain. \Cref{app:example-3d} writes out the model for the L-shaped prism.

\subsection{Parameter fitting}
\label{sec:parameter-fitting}

The model is fitted to boundary samples only. Both loss terms of \cref{eq:boundary-fit} are made dimensionless by a characteristic potential $U_0$ and length $L_0$: the Dirichlet term is divided by $U_0^2$ and the Neumann term multiplied by $L_0^2/U_0^2$; both scales are one on the sampled suite. Enrichment coefficients start at zero and exponents at the values of \cref{tab:dictionary} (the conical degree at $1/2$); exponents are fixed on Motz's problem, the cracked beam and the sampled suite. All parameters are optimized jointly with \textsc{Adam} followed by \textsc{L-BFGS}.

Every sample carries gradient on $\lambda$ through $\partial_\lambda w^\lambda=w^\lambda\operatorname{Log}w$, but far from a site the change is smooth and the rest of the model can absorb it, so we expect uniform boundary samples to carry little information about the exponent. Since the coefficients start at zero, an exponent also receives no gradient until its coefficient is nonzero. We therefore add, except on the sampled suite, samples whose distance to each junction or edge is drawn log-uniformly between $10^{-4}$ and $10^{-1}$ of the local length. Step counts and sample sizes are given in \cref{app:training-2d,app:training-3d}.

\subsection{Evaluation}
\label{sec:evaluation}

Where a reference solution $u^\star$ is available, we report the 
relative $L^2$ error
\begin{equation}
 \epsilon_u(X)=\left(\frac{\sum_{\bm{x}\in X}|u(\bm{x})-u^\star(\bm{x})|^2}
 {\sum_{\bm{x}\in X}|u^\star(\bm{x})|^2}\right)^{1/2},
\label{eq:evaluation-errors}
\end{equation}
estimated on a set $X$ of interior points. For the CZT detector, where 
the reference is numerical, we report the root-mean-square error instead. 
Where no closed-form solution exists, $u^\star$ is a singular-series or 
numerical reference, described with each problem. Unless stated otherwise, 
results are medians over three random seeds.
\section{Experiments in two dimensions}
\label{sec:experiments-2d}

\subsection{Problems}

We use six named benchmarks, each with a single singular site (\cref{app:benchmarks-2d}). Four have a closed-form solution: two circular sectors, an L-shaped domain and a slit disk, whose solution is the leading wedge mode of the opening angle \citep{grisvard1985}. The other two, Motz's problem \citep{motz1947} and the cracked beam \citep{lu2004}, are scored against their truncated singular expansion, whose leading coefficient $D_0$ is known to 17 significant digits \citep{lu2004}.

As an application we use a CdZnTe (CZT) drift strip radiation detector \citep{KUVVETLI2010486} geometry. CZT detectors are semiconductor gamma-ray detectors that operate at room temperature, and the electric potential inside the crystal determines how the charge released by a photon drifts to the electrodes. The electrodes on the crystal faces are held at fixed potentials and separated by insulating gaps, so every electrode edge is a junction between a Dirichlet and a Neumann boundary. A single detector therefore contains many singular sites of the same type. In 2D we use a cross-section of the detector, scored by the root-mean-square error against a numerical reference (\cref{app:czt-2d}). To test many sites on simpler geometries, we also use a suite of 24 randomly sampled problems with 4 to 12 sites each (\cref{app:suites-2d}).

\subsection{Results}

\paragraph{}{Baselines}
We compare three representations. The \gls{mlp} \gls{pinn} is a real-valued network trained on the Laplace residual at interior points and on the boundary misfit, with the weighting tuned per problem. The plain holomorphic model is \cref{eq:model2d} with $Q=0$, and the enriched model adds the terms of \cref{tab:dictionary}. The last two are the same network trained with the same budget and restarts, so they differ only in the enrichment.
\begin{table}[ht]
  \caption{Errors in 2D. All entries report $\epsilon_u$, median over 
  three seeds at the best budget of \cref{app:training-2d}, chosen by 
  the same rule for every method. \emph{Plain} and \emph{Enriched} are 
  the same network without and with the terms of \cref{tab:dictionary}.}
  \label{tab:benchmarks2d}
  \centering
  \small
  \begin{tabular}{lccc}
    \toprule
    Problem & \gls{mlp} \gls{pinn} & Plain & Enriched \\
    \midrule
    \multicolumn{4}{l}{\emph{Closed-form solution}} \\
    \quad Sector $3\pi/2$ & $9.5\times10^{-2}$ & $7.7\times10^{-4}$ & $\mathbf{3.3\times10^{-12}}$ \\
    \quad Sector $7\pi/4$ & $1.8\times10^{-1}$ & $1.9\times10^{-2}$ & $\mathbf{1.4\times10^{-16}}$ \\
    \quad L-shape         & $7.4\times10^{-2}$ & $7.5\times10^{-4}$ & $\mathbf{2.9\times10^{-16}}$ \\
    \quad Slit disk       & $1.8\times10^{-1}$ & $3.2\times10^{-1}$ & $\mathbf{2.8\times10^{-12}}$ \\
    \addlinespace
    \multicolumn{4}{l}{\emph{Series reference}} \\
    \quad Motz            & $1.7\times10^{-1}$ & $5.4\times10^{-4}$ & $\mathbf{2.5\times10^{-8}}$ \\
    \quad Cracked beam    & $8.6\times10^{-2}$ & $6.4\times10^{-4}$ & $\mathbf{8.6\times10^{-8}}$ \\
    \addlinespace
    \multicolumn{4}{l}{\emph{Application}} \\
    \quad CZT detector    & $7.8\times10^{-1}$ & $1.7\times10^{-2}$ & $\mathbf{1.0\times10^{-2}}$ \\
    \bottomrule
  \end{tabular}
\end{table}

The plain holomorphic model improves on the \gls{mlp} \gls{pinn} by a factor of 9 to 315 on five of the six benchmarks (\cref{tab:benchmarks2d}). The two configurations also differ in precision, architecture and boundary sampling (\cref{app:training-2d}), so this compares solver configurations rather than isolating exact harmonicity. The exception is the slit disk, where the plain model is worse than the \gls{mlp} \gls{pinn}: its solution vanishes on both slit faces but its transverse Cartesian derivative has opposite one-sided limits. An entire core extends smoothly across the slit and cannot reproduce this discontinuity, but the branch enrichment represents it directly.

\begin{figure}[ht]
  \centering
  \includegraphics[width=0.49\linewidth]{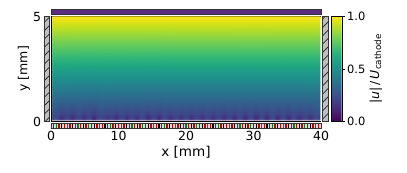}\hfill
  \includegraphics[width=0.49\linewidth]{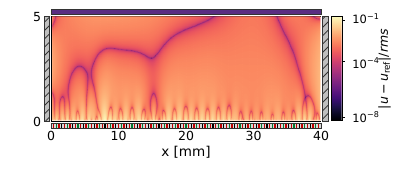}\\[1pt]
  \includegraphics[width=\linewidth]{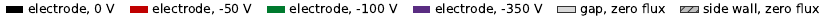}
  \vspace{-12pt}
  \caption{Enriched model on the 2D CZT detector. Left: normalized potential. Right: error against the reference}
  \label{fig:czt2d-slice}
\end{figure}

The enrichment gives a further 4 to 14 orders of magnitude, and the enriched model reaches roundoff on the $7\pi/4$ sector and the L-shape. On these closed-form benchmarks the solution is itself a dictionary term, so they test recovery of the supplied mode. On Motz's problem and the cracked beam it recovers $D_0$ to relative errors of $8.2\times10^{-7}$ and $2.0\times10^{-6}$, with the network supplying the correction that reduces the error of the same four singular terms fitted alone, $5.5\times10^{-3}$ and $2.2\times10^{-2}$ (\cref{app:benchmarks-2d}), to the reported values. With the same network and budget, tripling the \textsc{Adam} steps and boundary samples improves the plain model by less than an order of magnitude, while the enriched model improves by several (\cref{app:ablation-enrichment-2d}). On the sampled suite, the enrichment 
improves all 72 paired runs and reduces the median $\epsilon_u$ by a factor of $3.2\times10^{6}$ on the four closed-form instances.

On the detector (\cref{fig:czt2d-slice}), the plain model reduces the error of the \gls{mlp} \gls{pinn} by a factor of 46, and the enrichment reduces it by a further factor of 1.7. The enrichment gain is much smaller than on the benchmarks: the terms of many junctions must be fitted together with the smooth field between the electrodes. The enriched model is nevertheless 78 times more accurate than the \gls{mlp} \gls{pinn} on this detector cross-section.

On the pure Dirichlet benchmarks, the sampled right-hand side of \cref{eq:certificate} is close to the true error: wherever the error lies above roundoff, the largest interior error of the enriched model reaches 94\% to 100\% of it, making it a useful estimate when no reference solution is available.

\section{Experiments in three dimensions}
\label{sec:experiments-3d}

\subsection{Problems}

We group the problems as in 2D (\cref{app:problems-3d}). Four have a closed-form solution. The L-shaped prism has the cross-section of a re-entrant wedge and the solution $g(x,y)=r^{2/3}\sin(2\theta/3)$, which does not depend on the axial coordinate. Two further fields on the same prism vary along the edge,
\begin{equation}
 u_A(x,y,z)=(1+0.5z)g(x,y)+0.1xyz,\qquad
 u_B(\bm x)=u_A(\bm x)+\frac{0.1}{\|\bm x-(-2,-2,2)\|},
 \label{eq:targets-3d}
\end{equation}
and a smooth harmonic field on the unit cube has no singular site at all, so it tests the core alone. Two problems are scored against a series reference: the conical notch, whose solution is the vertex mode of \cref{eq:cone-3d}, and the extruded Motz problem, which adds zero-flux end faces to the planar problem and uses a 60-term singular series. As an application we use the full 3D CZT detector, posed in millimetres and volts, whose 238 electrode edges are all Dirichlet to Neumann junctions. It is scored against a graded finite-element reference whose own error is about $10^{-4}$.

\begin{table}[t]
  \caption{Relative error $\epsilon_u$ in 3D, median over three converged 
  seeds at each method's best schedule. The detector is scored against the 
  finite-element reference. $^\ddagger$Extension of WAN to mixed boundary 
  data.}
  \label{tab:benchmarks3d}
  \centering
  \small
  \setlength{\tabcolsep}{5pt}
  \begin{tabular}{lccccc}
    \toprule
    Problem & \gls{pinn} & Deep Ritz & WAN & Ballini et al. & Enriched \\
    \midrule
    \multicolumn{6}{l}{\emph{Closed-form solution}} \\
    \quad L-prism       & $8.7\times10^{-4}$ & $2.7\times10^{-3}$ & $5.6\times10^{-3}$ & $1.3\times10^{-3}$ & $\mathbf{6.9\times10^{-9}}$ \\
    \quad $u_A$         & $8.9\times10^{-4}$ & $4.5\times10^{-3}$ & $1.2\times10^{-2}$ & $1.9\times10^{-3}$ & $\mathbf{1.8\times10^{-4}}$ \\
    \quad $u_B$         & $1.1\times10^{-3}$ & $5.9\times10^{-3}$ & $1.2\times10^{-2}$ & $1.6\times10^{-3}$ & $\mathbf{1.6\times10^{-4}}$ \\
    \quad Smooth cube   & $1.5\times10^{-5}$ & $5.9\times10^{-4}$ & $2.2\times10^{-3}$ & $2.0\times10^{-4}$ & $\mathbf{4.6\times10^{-6}}$ \\
    \addlinespace
    \multicolumn{6}{l}{\emph{Series reference}} \\
    \quad Conical notch & $1.2\times10^{-3}$ & $1.0\times10^{-2}$ & $7.4\times10^{-2}$ & $4.1\times10^{-2}$ & $\mathbf{6.0\times10^{-9}}$ \\
    \quad Extruded Motz & $1.1\times10^{-2}$ & $4.1\times10^{-3}$ & $^\ddagger$$3.4\times10^{-1}$ & $5.8\times10^{-3}$ & $\mathbf{3.4\times10^{-6}}$ \\
    \addlinespace
    \multicolumn{6}{l}{\emph{Application}} \\
    \quad CZT detector  & $7.5\times10^{-2}$ & $2.7\times10^{-2}$ & $^\ddagger$$7.5\times10^{-2}$ & $2.4\times10^{-2}$ & $\mathbf{1.3\times10^{-2}}$ \\
    \bottomrule
  \end{tabular}
\end{table}

\subsection{Results}
\paragraph{Baselines}

We compare against four published mesh-free methods that impose the equation in different ways. The \gls{pinn} \citep{raissi2019} minimizes the strong-form residual, the Deep Ritz method \citep{e2018} minimizes the Dirichlet energy, and the weak adversarial network (WAN) of \citet{zang2020} minimizes the weak form against an adversarially trained test function. The holomorphic network of \citet{ballini2026} is exactly harmonic but has no singular terms. Each baseline follows its published protocol, with the authors' code where it runs on our problems, and is also run with its steps and boundary samples multiplied by 1.5, 2 and 3. We report the best implementation and budget of each method (\cref{app:training-3d}).

None of the four baselines reaches an error below $8\times10^{-4}$ on any problem with a singularity (\cref{tab:benchmarks3d}), and the exactly harmonic network of \citet{ballini2026} is not consistently more accurate than the \gls{pinn}. As in 2D, exact harmonicity is not sufficient without the singular terms. With them, the enriched model outperforms the best baseline by 5 or more orders of magnitude on the L-prism and the conical notch, whose solution is the supplied mode, and by three orders of magnitude on the extruded Motz problem.

\begin{figure}[ht]
  \centering
  \includegraphics[width=0.49\linewidth]{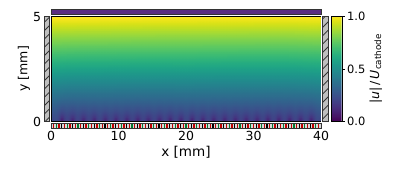}\hfill
  \includegraphics[width=0.49\linewidth]{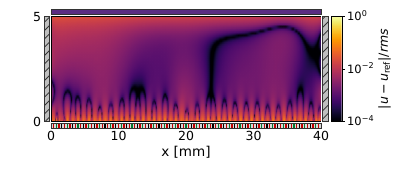}\\[1pt]
  \includegraphics[width=\linewidth]{images/depth_slice_legend.pdf}
  \vspace{-12pt}
  \caption{Enriched model on the 3D CZT detector, plane $z=21$\,mm. Left: normalized potential. Right: error against the reference, scaled by the root mean square of the reference solution.}
  \label{fig:czt3d-slice}
\end{figure}

The gain is smaller when the solution varies along the edge. The edge terms carry a constant amplitude, so the network must carry the axial variation of $u_A$ and $u_B$, and the margin falls to a factor of 5 to 7. On the smooth cube there is nothing to enrich, and the enriched model is three times more accurate than the \gls{pinn}.

On the detector, the enriched model reaches $1.3\times10^{-2}$, against $2.7\times10^{-2}$ for Deep Ritz and $7.5\times10^{-2}$ for the \gls{pinn}. The network of \citet{ballini2026} reaches $2.4\times10^{-2}$. The smaller margin may reflect the 238 edge terms fitted together with the field of the whole crystal, the unenriched strip ends, or conditioning. It is still the lowest potential error of the tested methods on this geometry.

\section{Conclusion}
\label{sec:conclusion}

We have shown that building the analytical structure of the Laplace equation 
directly into the hypothesis space is more effective than enforcing it through 
a residual penalty. By leveraging holomorphic network whose real part is harmonic by construction and using a set of singular enrichment terms we remove the need to use the interior points residual during training and handle difficult local behaviours that smooth neural methods cannot handle correctly. Furthermore, we extend the formulation to three dimensions through null complex projections. Together they reduce errors by orders of magnitude over both residual-based and exactly harmonic baselines across a range of 2D and 3D problems, with the largest gains where the dictionary captures the solution's singular behaviour exactly.

The advantage shrinks when many junctions share the domain, when edge 
amplitudes vary axially, or when the domain geometry falls outside the 
tested configurations. Several extensions follow naturally from the 
construction. A dictionary for general polyhedral vertices would extend 
the 3D results beyond the edge and conical cases. Edge terms with an 
affine axial amplitude would close the remaining gap on problems where 
the singular strength varies along the edge. More broadly, since the 
Laplace equation is linear, a single model with $N$ output heads sharing 
the same holomorphic core and enrichment structure could represent $N$ 
solutions on the same geometry simultaneously, as the junction structure 
and branch cuts are identical for all of them and only the smooth 
remainder and enrichment coefficients differ. For the CZT detector, where 
100 weighting potentials share the same electrode geometry, this would 
reduce training cost by an order of magnitude while exploiting the shared 
singular structure across all outputs.

\clearpage

\subsection*{Acknowledgements}

This work originated as a course project at the University of Amsterdam. We thank Chris Hoes, Remco Hogerwerf, Finn Govers, Storm Koper, Gerrit Janssen and Stef de Wildt, who were part of that project and contributed to the initial ideation and code development. This work was partially supported by the \textit{i-RASE} project, funded by the European Union under the Horizon Europe EIC Pathfinder Open programme (Grant Agreement No. 101130550).

\bibliography{iclr2027_conference}
\bibliographystyle{iclr2027_conference}

\appendix
\crefalias{section}{appendix}

\clearpage

\section{Harmonicity and branch cuts}
\label{app:harmonicity}

This section states when the model of \cref{sec:method} is harmonic and how the enrichment terms are placed so that the conditions hold.

\begin{lemma}[Harmonicity]
Assume that each network branch is entire and that every enrichment term is twice continuously differentiable and harmonic in $\Omega$. In 3D, also assume $a_k\cdot a_k=0$ for every projection direction. Then the representations in \cref{eq:model2d,eq:model-3d} are harmonic in $\Omega$.
\end{lemma}

\begin{proof}
In 2D, the real part of each holomorphic term is harmonic. In 3D, each branch has zero Laplacian by \cref{eq:chain}, and the centring and scaling of the input only multiply it by $L_0^{-2}$. The enrichment terms are harmonic by assumption, and taking real parts and summing preserves harmonicity by linearity of the Laplacian.
\end{proof}

The network branches satisfy the lemma automatically; the enrichment terms satisfy it only if their branch cuts stay outside the domain. We therefore orient each cut along a direction pointing out of the domain: the outward normal on a segment or face, and the exterior bisector at a corner (\cref{fig:sitekind_classification}). On the convex domains (the rectangles, the square and disk of the suite, the detector) such a cut meets the closure of the domain only at its site. On the sectors and the L-shape it runs through the removed wedge, on the slit disk it lies on the slit, and on the L-prism it lies in the removed quadrant. The cone mode has no cut and is singular only on the ray $\vartheta=\pi$, outside the domain. Site positions and orientations are fixed during training.

The cone mode is evaluated by the series $P_\nu(\cos\vartheta)={}_2F_1(-\nu,\nu+1;1;(1-\cos\vartheta)/2)$ truncated at 400 terms. Its argument is at most $0.854$ on the domain, so the truncation error is below $10^{-27}$.

The value-jump term of \cref{tab:dictionary} is $\arg w=\Re(-i\operatorname{Log}w)$, which carries the jump in the data and stays bounded at the junction. A general complex coefficient on $\operatorname{Log}w$ would add a radial part, since $\Re(c\operatorname{Log}w)=\Re(c)\ln|w|-\Im(c)\arg w$, and $\ln|w|$ is unbounded at the junction. This is the component excluded in \cref{sec:method}: with it, the error would be unbounded and \cref{eq:certificate} would not apply.

\begin{figure}[!ht]
  \centering
  \includegraphics[width=0.72\linewidth]{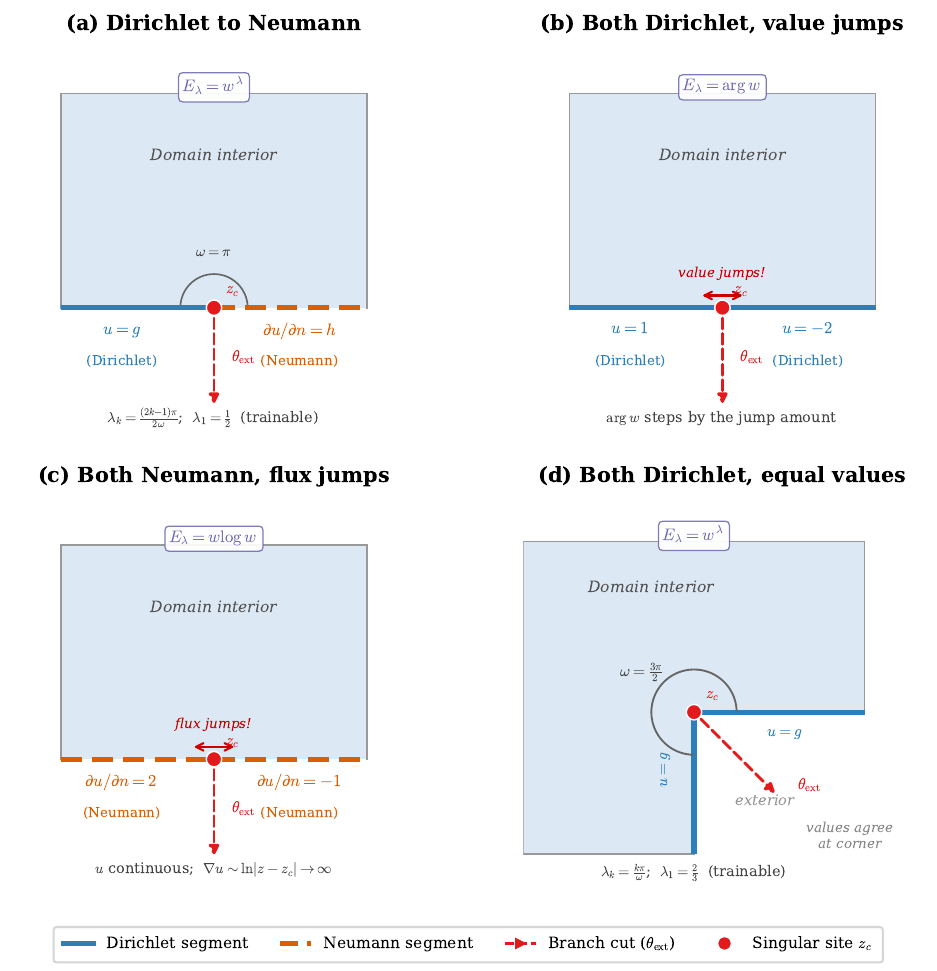}
  \caption{Enrichment terms and branch cuts for the rows of \cref{tab:dictionary}. \textbf{(a)} A Dirichlet to Neumann junction gives $w^\lambda$. \textbf{(b)} A Dirichlet value jump gives $\arg w$, which carries the jump. \textbf{(c)} A Neumann flux jump gives $w\operatorname{Log}w$. \textbf{(d)} A re-entrant corner with equal Dirichlet values gives $w^\lambda$; the equal-flux Neumann row has the same term. Each cut leaves the site along the exterior direction.}
  \label{fig:sitekind_classification}
\end{figure}

\section{Choice of architecture}
\label{app:architecture}

The 2D network size and optimizer settings come from the ablation of \cref{tab:ablation}. The 3D size, 16 branches of two hidden layers of width eight, was fixed after a preliminary single-seed screen of 135 configurations on four development problems, varying depth, width, activation, initialization, the number of branches and the treatment of the directions. We keep one configuration per dimension for all problems.

\section{Named benchmarks in two dimensions}
\label{app:benchmarks-2d}

\Cref{tab:benchmarks} defines the six benchmarks. Each has a single singular site; every other junction has an integer exponent, so \cref{tab:dictionary} emits no term there.

\begin{table}[ht]
  \caption{The named benchmarks. $\lambda$ is the leading exponent at the singular site. The reference coefficients are from \citet{lu2004}.}
  \label{tab:benchmarks}
  \centering
  \small
  \begin{tabular}{llll}
    \toprule
    Problem & Domain & $\lambda$ & Reference \\
    \midrule
    Sector $3\pi/2$ & $r<1$, $0<\theta<3\pi/2$ & $2/3$ & exact \\
    Sector $7\pi/4$ & $r<1$, $0<\theta<7\pi/4$ & $4/7$ & exact \\
    L-shape & $(-1,1)^2\setminus[0,1]\times[-1,0]$ & $2/3$ & exact \\
    Slit disk & $r<1$ cut along $0\le x<1$, $y=0$ & $1/2$ & exact \\
    Motz & $(-1,1)\times(0,1)$ & $1/2$ & $D_0=401.162453745234416$ \\
    Cracked beam & $(-1,1)\times(0,1)$ & $1/2$ & $D_0=540.56512271362749$ \\
    \bottomrule
  \end{tabular}
\end{table}

The first four carry Dirichlet data from the exact solution on the whole boundary. On the sectors, the data are zero on both radii and $r^\lambda\sin(\lambda\theta)$, $\lambda=\pi/\omega$, on the arc. On the L-shape, they are zero on the two edges meeting at the re-entrant corner and $r^{2/3}\sin(2\theta/3)$ on the outer edges. On the slit disk, they are zero on both faces of the slit and $\sin(\theta/2)$ on the rim, so that $u=r^{1/2}\sin(\theta/2)$.

Motz's problem and the cracked beam have mixed data. Both have $u=0$ on the bottom edge for $x<0$ and zero flux there for $x>0$. Motz's problem has zero flux on the top and left edges and $u=500$ on the right edge; the cracked beam has zero flux on the left and right edges and $u=500$ on the top edge. Both are scored against the series $\sum_{i=0}^{59}D_ir^{i+1/2}\cos((i+\frac12)\theta)$, which satisfies the bottom-edge conditions exactly, with coefficients fitted by least squares on the other three edges. It reproduces the published $D_0$ to relative errors of $7\times10^{-16}$ and $1.1\times10^{-15}$. Since the model's exponents are fixed on these problems, its leading coefficient compares directly with $D_0$. \Cref{tab:enrichment-only} shows how far the singular terms go without the network.

\begin{table}[ht]
  \caption{Least-squares fit of the first $n$ singular terms without the network: $\epsilon_u$ on 20,000 random interior points and the relative error of $D_0$. The model uses $n=4$.}
  \label{tab:enrichment-only}
  \centering
  \small
  \begin{tabular}{rcccc}
    \toprule
     & \multicolumn{2}{c}{Motz} & \multicolumn{2}{c}{Cracked beam} \\
    $n$ & $\epsilon_u$ & $D_0$ error & $\epsilon_u$ & $D_0$ error \\
    \midrule
    1  & $1.9\times10^{-1}$ & $7.9\times10^{-2}$ & $3.0\times10^{-1}$ & $1.5\times10^{-1}$ \\
    4  & $5.5\times10^{-3}$ & $4.4\times10^{-3}$ & $2.2\times10^{-2}$ & $1.1\times10^{-2}$ \\
    8  & $1.9\times10^{-4}$ & $3.5\times10^{-5}$ & $4.4\times10^{-4}$ & $8.6\times10^{-6}$ \\
    12 & $1.6\times10^{-5}$ & $5.1\times10^{-7}$ & $2.6\times10^{-5}$ & $2.8\times10^{-7}$ \\
    \bottomrule
  \end{tabular}
\end{table}

\begin{figure}[!ht]
    \centering
    \includegraphics[width=0.95\linewidth]{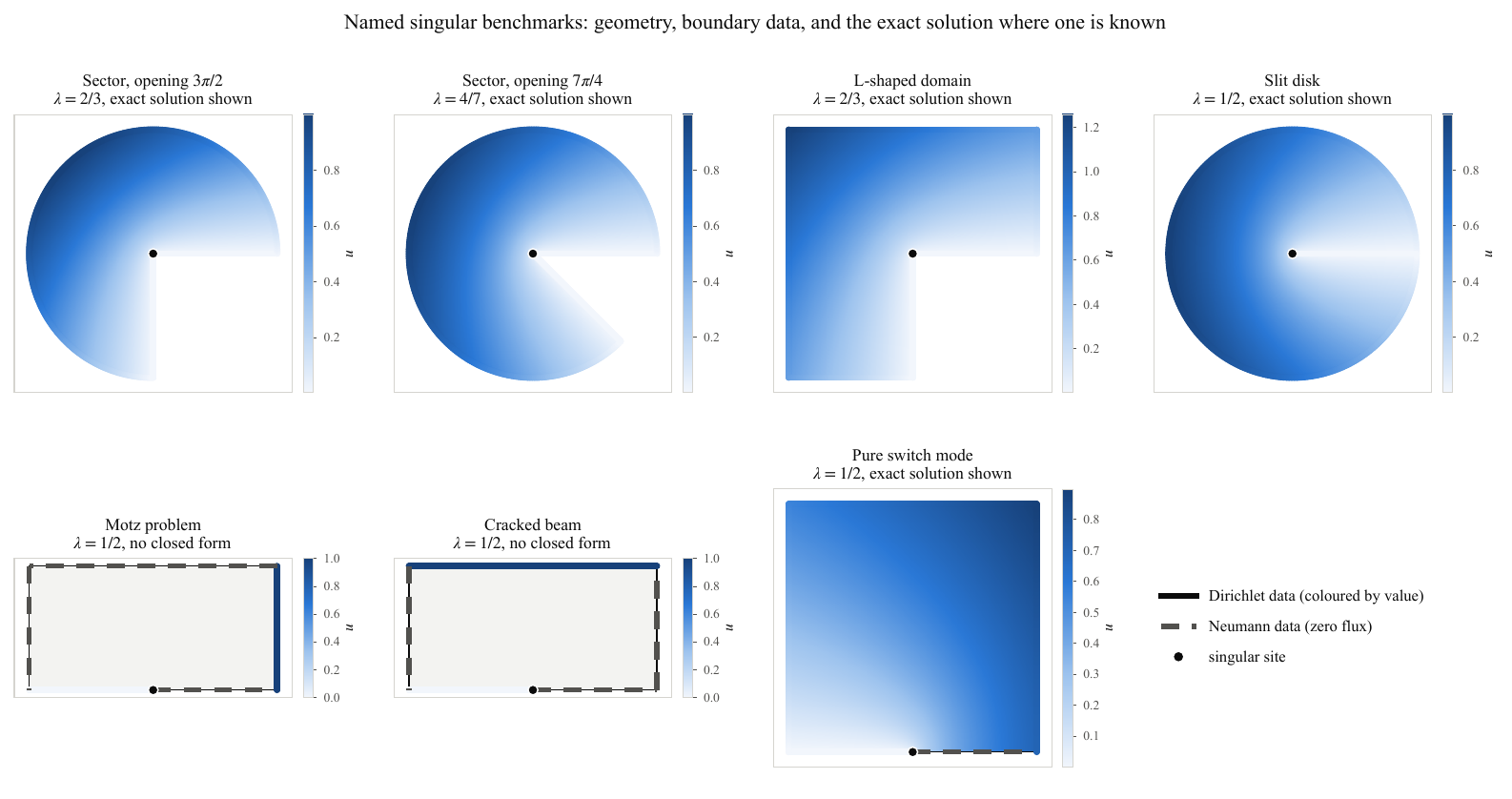}
    \caption{The six benchmarks of \cref{tab:benchmarks}, with the exact solution drawn where one exists.}
    \label{fig:benchmarks_overview}
\end{figure}

\section{CZT detector cross-section}
\label{app:czt-2d}

The 2D detector is a cross-section, 40\,mm wide and 5\,mm thick, of the crystal of \cref{app:problems-3d}. One face carries 24 anode strips at 0\,V, 0.2\,mm wide at a pitch of 1.6\,mm, with 25 groups of three drift strips at $-50$, $-100$ and $-50$\,V between them. The opposite face is a single cathode at $-350$\,V, and potentials are divided by $-350$\,V. The gaps and side walls carry zero flux, so each of the 198 strip edges is a Dirichlet to Neumann junction and receives one term $w^\lambda$, with $\lambda$ trained from $1/2$. The reference is a finite-difference solution, and \cref{tab:czt-2d} gives the errors on its grid.

\begin{table}[ht]
  \caption{Potential errors on the CZT cross-section, in units of 350\,V, median over three seeds.}
  \label{tab:czt-2d}
  \centering
  \small
  \begin{tabular}{lcc}
    \toprule
    Model & MSE & RMSE \\
    \midrule
    MLP PINN & $2.40\times10^{-1}$ & $4.90\times10^{-1}$ \\
    Plain    & $1.18\times10^{-4}$ & $1.09\times10^{-2}$ \\
    Enriched & $\mathbf{4.19\times10^{-5}}$ & $\mathbf{6.47\times10^{-3}}$ \\
    \bottomrule
  \end{tabular}
\end{table}

\section{Sampled suite}
\label{app:suites-2d}

The suite tests many sites on simple geometries. It has 24 instances, twelve on the unit square and twelve on the unit disk, in three groups of eight: Dirichlet data with value jumps, mixed Dirichlet and Neumann data, and configurations with many segments (\cref{fig:suite_eval}). The instances have 4 to 12 sites, with a median of 6.5, and are drawn with seed 20260713; the tuning suite of \cref{tab:ablation} is drawn separately with seed 20260804. The four Dirichlet-jump instances on the disk have a closed-form solution and are scored by $\epsilon_u$, the eight Dirichlet-jump instances also by the sampled right-hand side of \cref{eq:certificate}, and all instances by their boundary loss. Exponents are fixed at their dictionary values and no near-site samples are used.

\begin{figure}[!ht]
    \centering
    \includegraphics[width=0.72\linewidth]{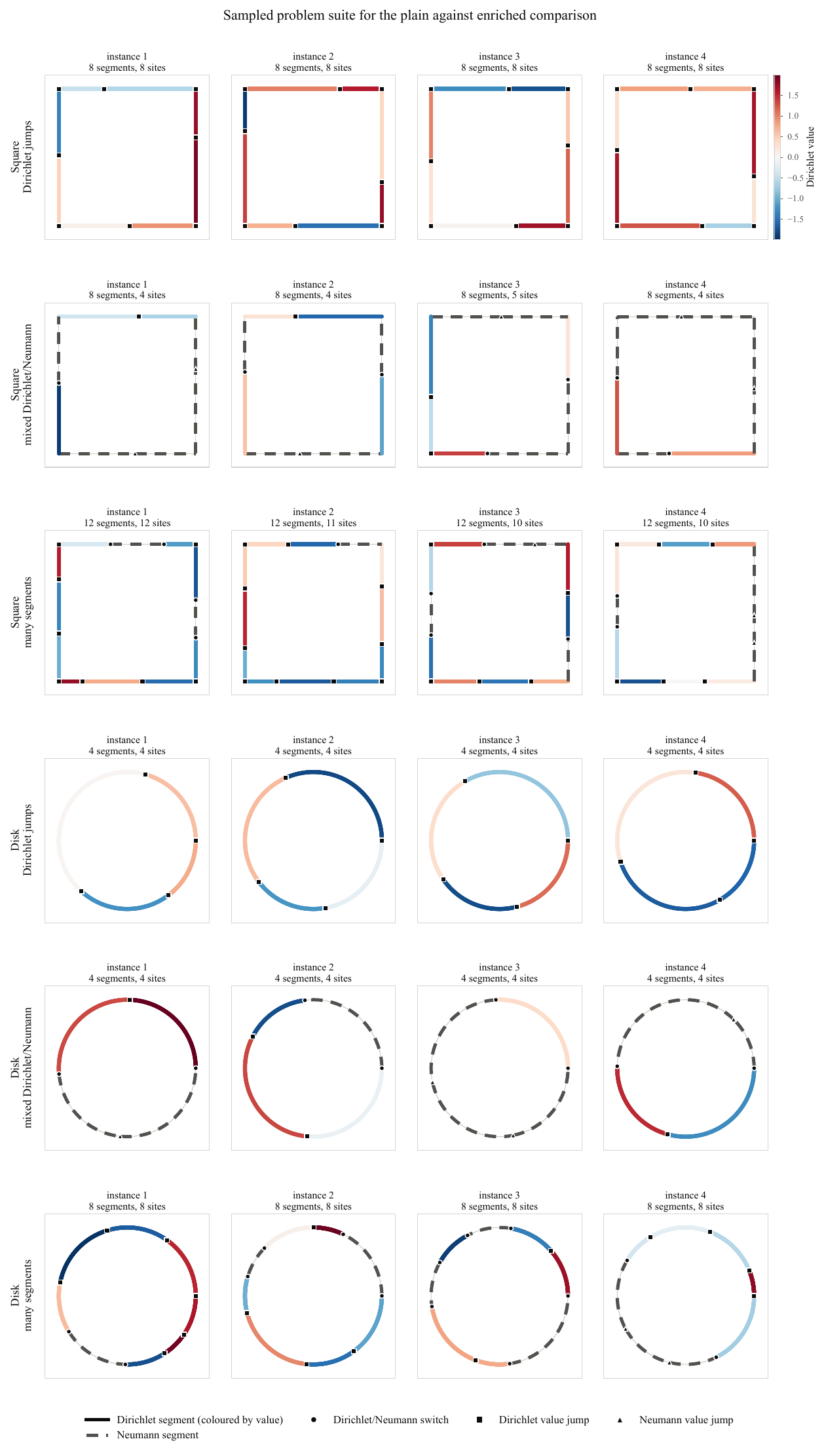}
    \caption{The 24 instances of the sampled suite. Dirichlet segments are coloured by their prescribed value, and Neumann segments are dashed. Markers show the sites the dictionary emits, with the shape giving the row of \cref{tab:dictionary}.}
    \label{fig:suite_eval}
\end{figure}

\section{Training in two dimensions}
\label{app:training-2d}

The model is implemented in JAX in double precision, since the reported errors go well below the $10^{-7}$ resolution of single precision. Weights are drawn from the scaled complex He initialization of \citet{calafa2024}. Training runs $4\times10^{4}$ full-batch Adam steps at learning rate $10^{-2}$ with cosine decay, followed by at most $6\times10^{3}$ L-BFGS iterations with a history of 100 \citep{liu1989}, and keeps the best of three restarts. Each segment receives 1,024 uniform samples, and each endpoint 128 further samples at log-uniform distances between $10^{-4}$ and $10^{-1}$ of the segment length. The sectors, the L-shape, the slit disk and the detector train their exponents; Motz's problem and the cracked beam fix them at $1/2$, $3/2$, $5/2$ and $7/2$.

The budget ladder multiplies the Adam steps and the samples per segment by 1, 1.5, 2 and 3, and each rung is run with seeds 0, 1 and 2. \Cref{tab:benchmarks2d} reports the best rung of each method. All other settings were chosen on the tuning suite (\cref{tab:ablation}), which is disjoint from the reported problems.

\paragraph{Baseline.} The MLP PINN is a real-valued tanh network with one hidden layer of width 128, trained in single precision on 5,000 interior points and $400b$ boundary points at budget $b$. Its loss is the Laplace residual plus the boundary misfit, with boundary weight 10 on the sectors and 100 elsewhere, minimized by Adam at learning rate $10^{-3}$ followed by 6,000 L-BFGS iterations, on the same budget ladder and seeds.

\begin{table}[ht]
  \caption{Ablation on the tuning suite (six instances, three seeds, 18 runs per row), changing one factor at a time from a starting configuration of width 32, depth 3, learning rate 0.003, 20\,000 Adam steps, 6\,000 L-BFGS steps, 512 points per segment, ladder depth 1 and fixed exponents. \emph{Ratio} is the median boundary loss relative to the starting configuration; below one is an improvement. \emph{Converged} counts runs with boundary loss below $10^{-6}$. The configuration of \cref{sec:model-2d} combines the five best changes.}
  \label{tab:ablation}
  \centering
  \small
  \begin{tabular}{lrrrr}
    \toprule
    Configuration & Boundary loss & Ratio & Converged & Time (s) \\
    \midrule
    starting configuration & $6.4 \times 10^{-5}$ & 1.00 & 3/18 & 402 \\
    \midrule
    depth 2 & $3.3 \times 10^{-6}$ & 0.03 & 7/18 & 280 \\
    learning rate 0.01 & $8.1 \times 10^{-6}$ & 0.14 & 7/18 & 381 \\
    Adam steps 40\,000 & $1.4 \times 10^{-5}$ & 0.41 & 5/18 & 644 \\
    width 16 & $1.4 \times 10^{-5}$ & 0.42 & 7/18 & 275 \\
    points per segment 1\,024 & $8.5 \times 10^{-6}$ & 0.43 & 6/18 & 645 \\
    depth 4 & $6.0 \times 10^{-5}$ & 0.84 & 1/18 & 512 \\
    L-BFGS steps 12\,000 & $3.8 \times 10^{-5}$ & 0.85 & 3/18 & 505 \\
    initialization empirical & $7.6 \times 10^{-5}$ & 0.95 & 3/18 & 381 \\
    trained exponents & $5.1 \times 10^{-5}$ & 0.98 & 3/18 & 457 \\
    ladder depth 2 & $6.9 \times 10^{-5}$ & 1.00 & 3/18 & 415 \\
    ladder depth 4 & $8.0 \times 10^{-5}$ & 1.00 & 3/18 & 487 \\
    L-BFGS steps 3\,000 & $7.2 \times 10^{-5}$ & 1.13 & 3/18 & 325 \\
    learning rate 0.001 & $6.5 \times 10^{-5}$ & 1.38 & 4/18 & 385 \\
    Adam steps 10\,000 & $1.4 \times 10^{-4}$ & 2.33 & 1/18 & 254 \\
    width 64 & $1.4 \times 10^{-4}$ & 2.36 & 3/18 & 1\,079 \\
    Adam steps 1\,000 & $3.8 \times 10^{-4}$ & 3.54 & 4/18 & 138 \\
    Adam steps 3\,000 & $5.3 \times 10^{-4}$ & 7.07 & 2/18 & 165 \\
    points per segment 256 & $1.6 \times 10^{-3}$ & 11.39 & 3/18 & 238 \\
    plain model & $4.0 \times 10^{-3}$ & 60.97 & 0/18 & 369 \\
    \bottomrule
  \end{tabular}
\end{table}

\section{Plain against enriched}
\label{app:ablation-enrichment-2d}

The plain and enriched models share the network, budget, restarts and seeds, and differ only in the singular terms. On the named benchmarks (\cref{fig:ladder_arms}), tripling the budget improves the plain model by less than an order of magnitude and the enriched model by several. On the sampled suite (\cref{fig:pairs}), all 72 paired runs improve with enrichment. The median paired ratio is $1.5\times10^{3}$ for the boundary loss, $1.1\times10^{7}$ for the sampled right-hand side of \cref{eq:certificate} on the eight Dirichlet-jump instances, and $3.2\times10^{6}$ for $\epsilon_u$ on the four closed-form instances. The gain is largest on value jumps and smallest on the many-segment instances.

To compare the error with the sampled bound of \cref{eq:certificate}, each Dirichlet segment is sampled at $2\times10^{4}$ points, excluding points within $10^{-3}$ of a site. On the two sectors, the largest error of the enriched model on a $400\times400$ interior grid reaches 94\% to 100\% of this bound; on the L-shape and the $7\pi/4$ sector both are at roundoff. For the plain model the ratio is 83\% to 90\% on all four.

\begin{figure}[!ht]
    \centering
    \includegraphics[width=\linewidth]{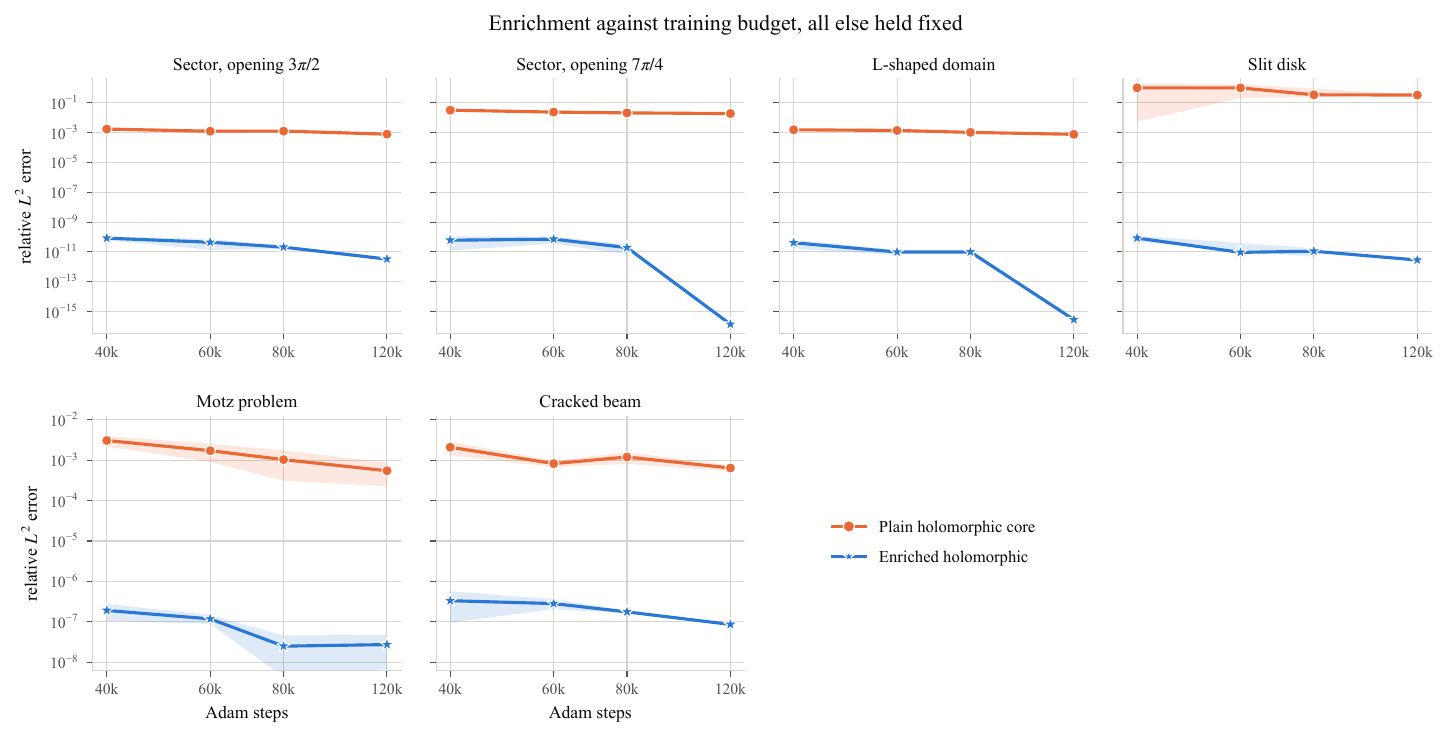}
    \caption{The plain and enriched models across the budget ladder on the named benchmarks, median over three seeds with the spread shaded.}
    \label{fig:ladder_arms}
\end{figure}

\begin{figure}[!ht]
    \centering
    \includegraphics[width=\linewidth]{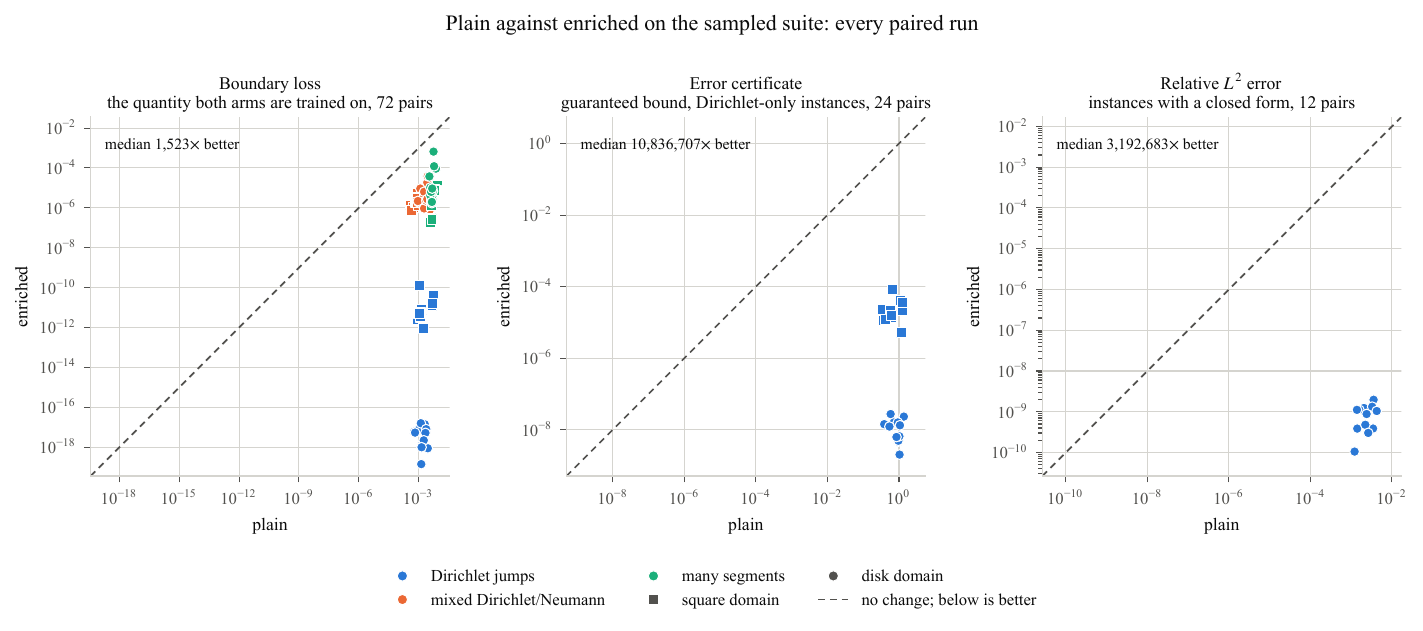}
    \caption{Plain against enriched over the sampled suite, one point per instance and seed. Points below the diagonal improved with enrichment; annotations give the median paired ratio.}
    \label{fig:pairs}
\end{figure}

\paragraph{Ladder depth.}
\Cref{tab:ladder} adds terms on the $7\pi/4$ sector, whose solution is the leading mode alone. The leading exponent stays within $2\times10^{-8}$ of its true value at every depth, while the surplus exponents cluster and carry amplitudes that cancel in pairs. The error stays near $10^{-10}$, so extra terms neither help nor harm here. These runs use budget one, which is why the depth-one error is above that of \cref{tab:benchmarks2d}.

\begin{table}[ht]
  \caption{Ladder depth on the $7\pi/4$ sector, $\lambda^\star=4/7$. $\lambda_1$ is the trained exponent closest to $\lambda^\star$, the gap is the smallest separation between trained exponents, and the spurious amplitude is the largest $|c_q|$ on a non-leading term.}
  \label{tab:ladder}
  \centering
  \small
  \begin{tabular}{rcccc}
    \toprule
    Terms & $|\lambda_1-\lambda^\star|$ & Smallest gap & Spurious $|c_q|$ & $\epsilon_u$ \\
    \midrule
    1 & $4.6\times10^{-11}$ & --                 & 0                  & $5.7\times10^{-11}$ \\
    2 & $2.3\times10^{-10}$ & $6.3\times10^{-1}$ & $1.9\times10^{-9}$ & $3.4\times10^{-11}$ \\
    3 & $7.4\times10^{-10}$ & $9.5\times10^{-8}$ & $3.7\times10^{-2}$ & $2.2\times10^{-11}$ \\
    4 & $8.8\times10^{-10}$ & $1.0\times10^{-7}$ & $1.8\times10^{-1}$ & $3.0\times10^{-11}$ \\
    6 & $2.0\times10^{-8}$  & $2.5\times10^{-6}$ & $1.0\times10^{-1}$ & $1.3\times10^{-10}$ \\
    8 & $1.9\times10^{-9}$  & $4.7\times10^{-7}$ & $2.2\times10^{-1}$ & $1.1\times10^{-10}$ \\
    \bottomrule
  \end{tabular}
\end{table}

\section{The 3D model on the L-shaped prism}
\label{app:example-3d}

On the L-shaped prism the only singular site is the re-entrant edge, with $\omega=3\pi/2$ and zero Dirichlet data on both walls. \Cref{tab:dictionary} assigns it the term $w^{2/3}$, and the model is
\begin{equation*}
 u(\bm{x})=\Re\left[\sum_{k=1}^{16}f_{\eta_k}(\zeta_k)+c\,w^{2/3}\right].
\end{equation*}
The exact solution $r^{2/3}\sin(2\theta/3)$ equals $\Re[-i\,w^{2/3}]$, so $c=-i$ represents it exactly and the branches need only represent zero. On $u_A$ and $u_B$ the singular amplitude varies along the edge, and the branches must carry that variation.

\section{Three-dimensional problems}
\label{app:problems-3d}

\paragraph{Closed-form solution.}
The L-prism is the L-shape of \cref{tab:benchmarks} extruded over $0<z<1$, with Dirichlet data from the exact solution on every face; $u_A$ and $u_B$ use the same geometry. The smooth cube is $[0,1]^3$ with Dirichlet data from $1/\|\bm x-(1.6,1.6,1.6)\|$, whose pole lies outside the cube.

\paragraph{Series reference.}
The conical notch is the unit ball with a cone of half-angle $\pi/4$ removed around the positive $z$ axis, so that, measured from the negative axis, the wall lies at $\vartheta=3\pi/4$. The data are zero on the wall and exact on the sphere, and the solution is \cref{eq:cone-3d} with $\nu=0.4630985617801066$; the trained degree, started at $1/2$, reaches $0.463099$. The extruded Motz problem has the planar Motz data on its cross-section and zero flux on its end faces. Its 60-term series reference changes by $5.5\times10^{-13}$ when reduced to 40 terms and matches the published $D_0$ to $7\times10^{-16}$.

\paragraph{CZT detector.}
The detector is a $40\times5\times40$\,mm crystal, posed in millimetres and volts. One face carries the anode and drift strips of \cref{app:czt-2d}, running the full length of the crystal, and the opposite face 20 cathode strips at $-350$\,V, 1.8\,mm wide at a pitch of 2\,mm. The gaps and the four sides carry zero flux, so each of the 119 electrodes has two Dirichlet to Neumann edges on a flat face, 238 in total, each with one edge term and $\lambda$ trained from $1/2$; the strip ends are not enriched. Zero flux idealizes an insulating surface, and space charge is neglected, so this is a numerical study of the operating potential rather than a validation of the device. The reference is a graded finite-element solution whose own error is about $10^{-4}$.

\section{Training and evaluation in three dimensions}
\label{app:training-3d}

\Cref{tab:protocol3d} lists the settings of the enriched model. The core has 3,104 real parameters; with the directions and enrichment terms the model has between 3,200 and 4,390, depending on the number of sites. \Cref{tab:seeds3d} gives the per-seed errors behind \cref{tab:benchmarks3d}.

\begin{table}[ht]
  \caption{Settings of the enriched model in 3D.}
  \label{tab:protocol3d}
  \centering
  \small
  \begin{tabular}{ll}
    \toprule
    Setting & Value \\
    \midrule
    Branches / hidden layers / width & $16/2/8$ \\
    Adam steps & 120,000, learning rate $0.01$, cosine decay \\
    Maximum L-BFGS iterations & 4,000 \\
    Uniform boundary samples & 1,536 per face (1,024 on the detector) \\
    Near-site samples & 128 per site, $10^{-4}$ to $10^{-1}$ of the local length \\
    Restarts & 3 per seed, selected on a fresh boundary sample \\
    Seeds & $0,1,2$ \\
    Precision & double \\
    Evaluation & 10,000 interior points per problem \\
    \bottomrule
  \end{tabular}
\end{table}

\begin{table}[ht]
  \caption{Per-seed relative error $\epsilon_u$ of the enriched model in \cref{tab:benchmarks3d}.}
  \label{tab:seeds3d}
  \centering
  \small
  \begin{tabular}{lccc}
    \toprule
    Problem & Seed 0 & Seed 1 & Seed 2 \\
    \midrule
    L-prism       & $1.3\times10^{-8}$ & $1.7\times10^{-9}$ & $6.9\times10^{-9}$ \\
    $u_A$         & $1.8\times10^{-4}$ & $1.9\times10^{-4}$ & $1.6\times10^{-4}$ \\
    $u_B$         & $1.8\times10^{-4}$ & $1.6\times10^{-4}$ & $1.6\times10^{-4}$ \\
    Smooth cube   & $4.4\times10^{-6}$ & $4.9\times10^{-6}$ & $4.6\times10^{-6}$ \\
    Conical notch & $4.5\times10^{-9}$ & $6.0\times10^{-9}$ & $6.1\times10^{-9}$ \\
    Extruded Motz & $3.4\times10^{-6}$ & $2.1\times10^{-6}$ & $4.6\times10^{-6}$ \\
    \bottomrule
  \end{tabular}
\end{table}

\paragraph{Baselines.}
Each baseline runs under its own published protocol: the architecture, optimizer, sampling and step counts of its paper, with the authors' code where it runs on our problems. It is also run with its steps and boundary samples multiplied by 1.5, 2 and 3, and \cref{tab:benchmarks3d} reports its lowest three-seed median. Deep Ritz uses three residual blocks of width ten with cubic ReLU activations and a boundary penalty of $10^3$. WAN is formulated for Dirichlet data; on the extruded Motz problem and the detector, the zero-flux condition enters through the weak form, and we mark these entries as an extension. Baselines that require unit scales are fitted on the detector with lengths divided by 5\,mm and potentials by 350\,V, and mapped back before scoring.

\end{document}